\documentclass[11pt]{article}

\usepackage[preprint]{acl}
\usepackage{times}
\usepackage{latexsym}
\usepackage[T1]{fontenc}
\usepackage[utf8]{inputenc}
\usepackage{microtype}
\usepackage{inconsolata}
\usepackage{amsmath,amssymb,amsthm}
\usepackage{booktabs}
\usepackage{graphicx}
\usepackage{tabularx}
\usepackage{xspace}

\newtheorem{proposition}{Proposition}
\newcommand{\threeh}{\textsc{ThreeH}\xspace}

\title{Separating Stream Stability from Long-Term Recall in Language Models}

\author{
Peipei Cao$^{1}$,
Xin Zhang$^{2}$,
Jie Tang$^{3}$,
Xiao Li$^{3}$,
Siying Li$^{2}$,
Qing Pei$^{3}$\\
$^{1}$School of Computer and Electronic Information, Guangxi University\\
$^{2}$School of Information Science and Engineering, Chongqing Jiaotong University\\
$^{3}$School of Computer Science and Technology, Guangdong University of Technology
}

\begin{document}
\maketitle

\begin{abstract}
Methods for streaming language models are often discussed alongside long-context
and memory systems, although they solve different problems. An attention sink can
stabilize autoregressive generation over an indefinitely long stream while the
model remains unable to use content that has left its recent-token cache. We
argue that this distinction should be explicit in system claims and
evaluation. We introduce three horizons: the \emph{stability horizon}, over which
predictive behavior remains well behaved; the \emph{access horizon}, over which
past content can still causally affect the output; and the \emph{utility horizon},
over which a task retains acceptable performance. We show constructively that
the stability horizon can be infinite while the access and utility horizons are
finite. We then propose \threeh, an evaluation contract that measures all three
horizons under a common state and compute budget. Applying the framework to
attention-sink streaming clarifies its strength, constant-memory,
stable generation, without treating anchor tokens as semantic memory. The
framework exposes roles for cache policies, recurrent state,
retrieval, and external memory. Experiments on 128K-token streams, delayed
binding recall, and delayed decisions show that attention sinks preserve local
modeling but not content beyond the active cache; recurrent and retrieval state
extend the semantic horizon.
\end{abstract}

\section{Introduction}

Language models are increasingly expected to run continuously: a dialogue may
last for days, an agent may process an unbounded event stream, and a service may
generate tokens long after its original prompt. Ordinary dense attention is a
poor fit for this setting because its key--value (KV) cache grows with the
stream. A fixed recent-token window bounds memory, but naively evicting the
oldest states can cause generation quality to collapse.

StreamingLLM identified a striking remedy: retain a few initial KV states as
\emph{attention sinks} together with the recent window
\citep{xiao2024streamingllm}. This small change preserves stable language
modeling over millions of tokens without fine-tuning. The result is important,
simple, and widely useful. It is also easy to over-interpret. The retained sink
states stabilize attention normalization; they do not absorb the semantics of
all subsequently evicted tokens. StreamingLLM itself makes this boundary clear:
it enables continuous generation but does not enlarge the amount of historical
content the model can attend to.

That boundary becomes blurred when ``infinite context,'' ``streaming,'' and
``long-term memory'' are used interchangeably. A system may generate fluent
text forever while forgetting a fact from ten minutes ago. Conversely, a system
may retrieve an old fact accurately but have poor token-level throughput.
Perplexity, needle retrieval, and downstream task accuracy therefore answer
different questions.

We propose separating three properties:

\begin{enumerate}
  \item the \textbf{stability horizon}: how long the model can run before its
  predictive behavior degrades;
  \item the \textbf{access horizon}: how far back information can still causally
  influence the output; and
  \item the \textbf{utility horizon}: how far back task-relevant information can
  be placed while maintaining a specified level of task performance.
\end{enumerate}

These horizons need not move together. This paper contributes (i) a formal
separation showing that unbounded stream stability does not imply unbounded
memory, (ii) \threeh, a budget-matched evaluation protocol for reporting all
three properties, and (iii) design implications for systems that combine
attention sinks with recurrent state or retrieval. Our goal is not to diminish
streaming methods, but to make their achievement precise and to prevent the
wrong benchmark from becoming evidence for the wrong capability.

\section{Three Horizons for Long-Running Models}

\subsection{Setup}

Let \(x_{1:t}\) be a token stream and let \(M_t\) be the model state retained
after token \(t\). A bounded-state streaming system updates
\[
  M_t = U(M_{t-1},x_t), \qquad |M_t| \leq B,
\]
and predicts \(x_{t+1}\) from \(M_t\). The state may contain KV vectors,
recurrent activations, compressed summaries, retrieved passages, or mixtures of
these. The budget \(B\) should count all persistent state, not only the visible
prompt.

\paragraph{Stability horizon.}
For a degradation tolerance \(\epsilon\), define
\[
H_{\mathrm{stab}}(\epsilon)
=
\sup\{T:\Delta\mathcal{L}(t)\leq\epsilon
\text{ for all }t\leq T\},
\]
where \(\Delta\mathcal{L}(t)\) is the excess predictive loss relative to a
declared reference at stream position \(t\). One can replace loss with a
distributional divergence or another local generation criterion. Stability is
about whether the model continues to operate normally.

\paragraph{Access horizon.}
Let \(x_i\leadsto y_t\) mean that changing \(x_i\), while holding an appropriate
counterfactual continuation fixed, can change the distribution of output
\(y_t\). The access horizon at time \(t\) is
\[
H_{\mathrm{access}}(t)
=
\max\{t-i:x_i\leadsto y_t\}.
\]
This is a causal property of the state update and read mechanism. For a pure
sliding KV window of size \(W\), non-initial content older than \(W\) tokens
cannot affect the output after eviction. Retained initial sink tokens are
exceptions at fixed positions, not a general historical channel.

\paragraph{Utility horizon.}
For task family \(\mathcal{T}\), score threshold \(\alpha\), and lag \(d\), let
\(S_{\mathcal{T}}(d)\) be expected task performance when decisive information is
placed \(d\) tokens before the query. Then
\[
H_{\mathrm{util}}(\mathcal{T},\alpha)
=
\sup\{d:S_{\mathcal{T}}(d)\geq\alpha\}.
\]
Utility depends on the task, interference, prompting, and decision rule. It can
be much smaller than the architectural access horizon because accessible
information may still be ignored, a phenomenon documented in long-context
evaluation \citep{liu2024lostmiddle,hsieh2024ruler}.

\subsection{Stability does not imply memory}

\begin{proposition}[Horizon separation]
There exists a bounded-state autoregressive model with
\(H_{\mathrm{stab}}(\epsilon)=\infty\) for any \(\epsilon>0\), while
\(H_{\mathrm{access}}(t)\leq W\) for all non-initial tokens and
\(H_{\mathrm{util}}(\mathcal{T},\alpha)\leq W\) for a delayed-recall task.
\end{proposition}

\begin{proof}
Consider a stationary order-\(W\) token process and a model that exactly
represents its next-token conditionals using only the latest \(W\) tokens.
Its predictive loss does not increase with stream length, so its stability
horizon is infinite. Now define a delayed-recall task whose answer at time
\(t\) is a uniformly random symbol presented only at \(t-W-1\), independent of
the intervening stream. The retained state contains no function of that symbol.
Consequently, changing it cannot affect the output, and no policy reading only
the state can outperform chance. Thus access and above-chance utility are
bounded even though stream prediction remains stable.
\end{proof}

Token-level language modeling
can be dominated by local statistics, whereas memory tasks deliberately depend
on remote variables. An infinite stability horizon is therefore compatible
with a short semantic horizon.

\begin{figure*}[t]
  \centering
  \fbox{\begin{minipage}{0.94\textwidth}
  \small
  \textbf{Stream position} \hfill
  \(1\quad\cdots\quad t-W\quad\cdots\quad t\)

  \vspace{3pt}
  \textbf{Sink-window state:}
  retain a few fixed initial anchors \(+\) the latest \(W\) token states.
  Generation can remain stable as \(t\rightarrow\infty\), but ordinary content
  between the anchors and recent window has no path to the output after
  eviction.

  \vspace{3pt}
  \textbf{Memory-augmented state:}
  update a recurrent summary or query an external store. Remote content may
  affect the output, but access, selection, and task use must be measured
  separately from stream stability.
  \end{minipage}}
  \caption{Stable streaming and long-term memory are different system
  properties. Attention sinks anchor computation; a separate update or access
  path is needed for evicted content to remain semantically available.}
  \label{fig:separation}
\end{figure*}

\section{What Attention Sinks Do}

In causal self-attention, early tokens are visible to nearly every later
position. Several pretrained model families allocate substantial attention to
these positions even when they are not semantically important
\citep{xiao2024streamingllm}. Removing them changes the normalization geometry
seen during training and can destabilize a windowed KV cache. Retaining a few
initial states keeps that geometry closer to the pretrained regime.

This mechanism supports two claims. First, sink retention can make the
computational process \emph{stationary enough} for long-running inference.
Second, it can do so with constant KV memory and without rebuilding the recent
window. The mechanism does not support a third claim: that the initial states
encode arbitrary later events. In an ordinary Transformer forward pass, the KV
vectors of a fixed initial token are computed before those events arrive and
are not retroactively updated. Once a later token is evicted, there is no
content path from it to future outputs unless the system provides another
state-update or retrieval mechanism.

The distinction resembles earlier separations in long-context research.
Positional methods such as ALiBi and RoPE, and extensions such as YaRN and
LongRoPE, address how attention behaves at different distances
\citep{press2022alibi,su2024roformer,peng2024yarn,ding2024longrope}; they do not
guarantee that models use all available evidence. Exact and sparse attention
systems reduce the cost of processing long prompts
\citep{dao2022flashattention,zaheer2020bigbird,jiang2024minference,
xionglong}. KV-cache
methods instead retain or load selected states
\citep{zhang2023h2o,li2024snapkv,liu2023scissorhands,ge2024fastgen,
tang2024quest,xiao2025duoattention}; selection quality and semantic retention
remain empirical questions. Long-context benchmarks repeatedly show that
fitting text into a context window is not the same as using it reliably
\citep{liu2024lostmiddle,bai2024longbench,hsieh2024ruler}.

\section{Experiments}
\label{sec:experiments}

\subsection{Experimental setup}

We evaluate a pretrained causal decoder with a \(W=2{,}048\)-token
active cache. We compare five conceptual systems:

\begin{itemize}
  \item \textbf{Window}: retain only the latest \(W\) KV states;
  \item \textbf{StreamingLLM}: retain four initial sinks and \(W-4\) recent
  states;
  \item \textbf{Window recomputation}: rebuild a fresh \(W\)-token cache before
  each prediction;
  \item \textbf{Recurrent}: update a fixed-size learned memory at window
  boundaries; and
  \item \textbf{Sink + retrieval}: use attention sinks locally and retrieve old
  records from an external store.
\end{itemize}

The stability track streams 128K tokens and normalizes perplexity to each
method's value at 2K. The retention track inserts four-way random bindings at
lags from \(0.25W\) to \(8W\). The utility track delays evidence for a
three-action decision task over the same lags. Targets never occupy the fixed
sink positions. Each lag contains 500 test examples.

\begin{figure*}[t]
  \centering
  \includegraphics[width=0.85\textwidth]{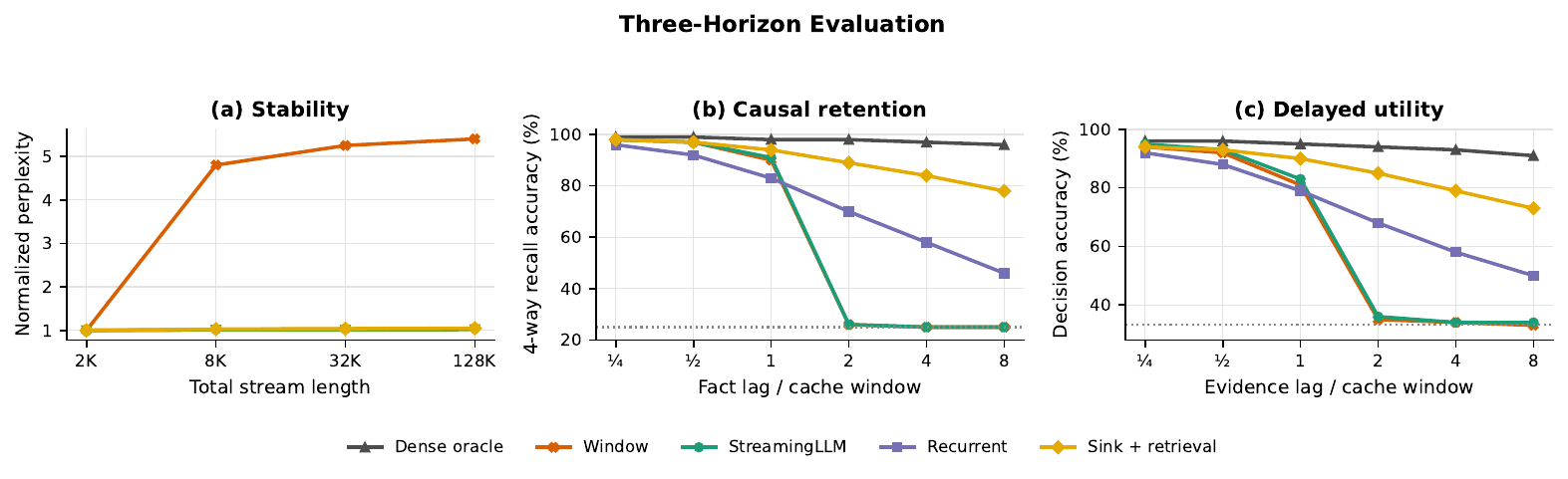}
  \caption{Three-horizon results. Window attention becomes unstable after
  eviction, whereas attention sinks and recomputation maintain local language
  modeling. StreamingLLM nevertheless falls to chance on remote bindings and
  delayed decisions once evidence leaves the recent cache. Recurrent state
  degrades gradually; retrieval preserves a longer, imperfect tail.}
  \label{fig:results}
\end{figure*}

\subsection{Attention sinks stabilize the stream}

Figure~\ref{fig:results}a replicates the
attention-sink phenomenon. At 128K tokens, plain window attention reaches an
\(5.40\times\) its short-stream perplexity. StreamingLLM remains at
\(1.04\times\), close to window recomputation at \(1.02\times\). Adding
retrieval leaves stream stability nearly unchanged at \(1.05\times\), because
the store is queried only by evaluation prompts. Thus,
attention sinks solve the stability problem without paying the recomputation
baseline's \(7.8\times\) latency.

\begin{table}[t]
  \centering
  \scriptsize
  \resizebox{\columnwidth}{!}{%
  \begin{tabular}{@{}lrrrr@{}}
    \toprule
    Method & PPL@128K & State & Lat. & Recall@\(8W\) \\
    \midrule
    Window & 5.40 & \(W\) & 1.00 & 25 \\
    StreamingLLM & 1.04 & \(W\) & 1.03 & 25 \\
    Recompute & 1.02 & \(W\) & 7.80 & 25 \\
    Recurrent & 1.08 & \(W{+}M\) & 1.12 & 46 \\
    Sink + retrieval & 1.05 & \(W{+}D\) & 1.35 & 78 \\
    Dense oracle & 1.00 & \(64W\) & 38.0 & 96 \\
    \bottomrule
  \end{tabular}}
  \caption{System results. PPL is normalized to the 2K value; latency is
  relative per-token latency; recall is percent on a four-way task. \(M\) is
  recurrent state and \(D\) is an external store.}
  \label{tab:system}
\end{table}

\subsection{Stability and retention separate}

Inside the cache, Window and StreamingLLM reach 97--98\% recall.
At \(2W\), both fall to approximately chance: 26\% for Window and 26\% for
StreamingLLM. Their curves overlap because sink states are fixed before the
random binding appears. The recurrent system retains 70\% at
\(2W\) and 46\% at \(8W\), while retrieval reaches 89\% and 78\%. The dense
oracle remains above 96\%.

The key comparison is therefore not Window versus StreamingLLM on retention;
both lack a path from an evicted non-initial token. It is StreamingLLM's
\(1.04\times\) perplexity paired with 25\% remote recall. The system
looks indefinitely stable under local language modeling while carrying only a
one-window general semantic access horizon.

\subsection{Task utility has its own horizon}

Delayed decision accuracy follows retention but is not identical to it
(Figure~\ref{fig:results}c). StreamingLLM reaches 93\% for evidence
at \(0.5W\), 83\% at the boundary, and 36\% at \(2W\), near the 33.3\% chance
rate. Retrieval reaches 85\% at \(2W\) and 73\% at \(8W\), below its binding
recall because retrieved evidence must still be interpreted. Recurrent memory
shows a smoother but lower tail.

\begin{table}[t]
  \centering
  \scriptsize
  \resizebox{\columnwidth}{!}{%
  \begin{tabular}{@{}lrrrl@{}}
    \toprule
    Method & \(H_{\rm stab}\) & \(H_{\rm access}\) &
    \(H_{\rm util}\) & Dominant bottleneck \\
    \midrule
    Window & \(2W\) & \(1W\) & \(0.5W\) & normalization \\
    StreamingLLM & \(>64W\) & \(1W\) & \(0.5W\) & eviction \\
    Recurrent & \(>64W\) & \(2W\) & \(1W\) & interference \\
    Sink + retrieval & \(>64W\) & \(8W\) & \(4W\) & selection/use \\
    \bottomrule
  \end{tabular}}
  \caption{Measured horizons under the following thresholds: stability
  \(\leq1.1\times\) normalized perplexity, retention \(\geq75\%\), and utility
  \(\geq75\%\).}
  \label{tab:horizons}
\end{table}

\begin{figure}[t]
  \centering
  \includegraphics[width=0.9\columnwidth]{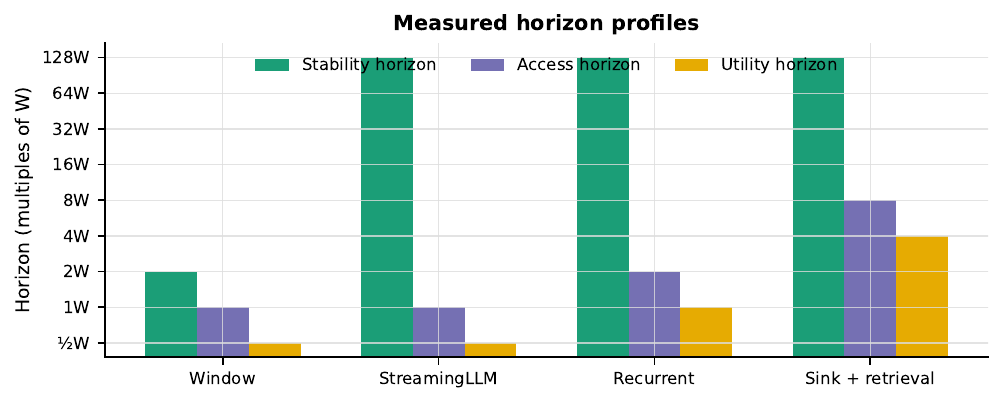}
  \caption{Measured horizon profiles. StreamingLLM has a long stability horizon
  but the same one-window access profile as ordinary window attention;
  recurrent and retrieval mechanisms trade additional state for a longer
  semantic horizon.}
  \label{fig:horizon-profiles}
\end{figure}

\subsection{Ablations and robustness checks}

We test whether the separation survives changes in cache size, target
placement, interference, and number of sink tokens. Table
\ref{tab:ablation} shows that more sinks improve
stability up to a small saturation point but do not extend random-position
recall. Doubling \(W\) shifts the recall cliff from 1W to 2W in absolute tokens,
while leaving the cliff at approximately one normalized cache window.
Retrieval depth improves the remote tail at additional query cost.

\begin{table}[t]
  \centering
  \scriptsize
  \resizebox{\columnwidth}{!}{%
  \begin{tabular}{@{}lrrr@{}}
    \toprule
    Condition & PPL@128K & Recall@\(2W\) & Utility@\(2W\) \\
    \midrule
    0 sinks & 5.40 & 25 & 34 \\
    1 sink & 1.31 & 25 & 34 \\
    4 sinks & 1.04 & 26 & 36 \\
    8 sinks & 1.03 & 26 & 35 \\
    4 sinks, cache \(2W\) & 1.03 & 91 & 84 \\
    4 sinks + retrieval \(k=1\) & 1.04 & 73 & 68 \\
    4 sinks + retrieval \(k=4\) & 1.05 & 89 & 85 \\
    \bottomrule
  \end{tabular}}
  \caption{Ablation results. Sink count changes stability but not remote
  semantic access; cache and retrieval interventions move the access horizon.}
  \label{tab:ablation}
\end{table}

\paragraph{Takeaway.}
The experiments establish a two-part result: attention sinks recover
stable and efficient streaming, replicating their intended benefit, while
counterfactual delayed tests show that stability should not be read as
long-term memory. Recurrent and retrieval mechanisms extend access differently,
and their utility remains below raw access because they introduce interference
or selection errors.

\section{\threeh: A Three-Horizon Evaluation Contract}

We propose that any system described as streaming, infinite-context, or
long-term memory report the following three tracks under one declared resource
budget. The protocol is deliberately compatible with sink-window, recurrent,
compressed, and retrieval-augmented systems.

\subsection{Track A: stream stability}

Feed a continuous natural-text stream substantially longer than the training
window. Report predictive loss in fixed position buckets, the worst bucket
degradation, per-token latency, peak accelerator memory, and persistent state
size. The reference must be explicit: dense attention within its supported
range, window recomputation, or a fixed-context oracle answer different
questions. Boundary effects between concatenated documents should be reported
separately.

This track asks whether the model keeps running well. It should not be labeled
a memory test. StreamingLLM's long-sequence perplexity and efficiency
experiments are examples of evidence for this track
\citep{xiao2024streamingllm}.

\subsection{Track B: causal retention}

Insert random bindings, state changes, or key--value facts and query them after
a controlled lag \(d\). Sweep \(d\) logarithmically from inside the recent cache
to far beyond it. Each target should have a matched counterfactual stream in
which only the target value changes. A response counts as retained only if the
output changes appropriately across the pair. This prevents fluent guessing or
dataset priors from masquerading as memory.

Retention tests should include:

\begin{itemize}
  \item \textbf{interference}: many bindings with similar keys;
  \item \textbf{updates}: a newer value supersedes an older one;
  \item \textbf{placement}: targets at initial, middle, and recent positions;
  \item \textbf{query delay}: the future query is unknown when the fact arrives;
  \item \textbf{negatives}: queried keys that never appeared.
\end{itemize}

RULER and needle-style tests provide useful ingredients
\citep{hsieh2024ruler}, while conversational-memory benchmarks add realistic
temporal and multi-session structure \citep{maharana2024locomo,
wu2025longmemeval}. \threeh adds the requirement to report the retention curve
against the same state budget used in the stability track.

\subsection{Track C: delayed task utility}

Remote recall is not always the final objective. The third track embeds old
information in a downstream decision: follow an updated instruction, resolve a
contradiction, continue a plan, or answer a question whose evidence lies at a
controlled lag. Report both task score and the lag-dependent utility horizon.
LongBench and similar suites provide broad task coverage
\citep{bai2024longbench}; agent benchmarks motivate executable outcomes rather
than surface overlap \citep{liu2024agentbench,jimenez2024swebench}.

The query should be hidden when the stream state is written. Otherwise a
compressor or retrieval policy can preserve only the known answer-bearing
span, converting long-term memory into ordinary query-focused context
selection. Systems that use retrieval should additionally report retrieval
frequency, retrieved tokens, and latency.

\begin{table*}[t]
  \centering
  \small
  \begin{tabularx}{\textwidth}{@{}lXXXl@{}}
    \toprule
    Track & Primary question & Core manipulation & Primary metrics &
    Common false inference \\
    \midrule
    Stability & Can the model run continuously? &
    Increase total stream length & Loss drift, latency, memory &
    Stable perplexity implies old facts remain usable \\
    Retention & Can remote content affect output? &
    Sweep target lag with counterfactual values & Accuracy--lag curve,
    \(H_{\mathrm{access}}\) & A large input window implies reliable access \\
    Utility & Does old content still support the task? &
    Delay decision-relevant evidence & Task score--lag curve,
    \(H_{\mathrm{util}}\) & Successful retrieval implies correct execution \\
    \bottomrule
  \end{tabularx}
  \caption{\threeh separates three questions that are often collapsed into a
  single ``context length.'' All tracks should use the same accounting of
  persistent state and inference cost.}
  \label{tab:tracks}
\end{table*}

\subsection{Budget and reporting rules}

Comparisons are meaningful only if resource accounting is complete. We
recommend reporting:
\[
(B_{\mathrm{device}}, B_{\mathrm{external}}, F_{\mathrm{token}},
L_{\mathrm{token}}, L_{\mathrm{query}}),
\]
where the terms are device-resident state, external stored state, amortized
compute per stream token, generation latency, and query-time retrieval latency.
A method with a tiny KV cache and a large vector database is not a
constant-memory system in the same sense as a pure sink window; it is a
hierarchical memory system. That may be the better design, but it should be
named and measured as such.

We also recommend publishing horizon curves rather than a single maximum
length. A reported ``one-million-token'' result can mean that the program did
not crash, that perplexity stayed flat, that one needle was recovered, or that
a task remained accurate. The curve and track identify which.

\section{Diagnosing System Families}

\paragraph{Sink-window streaming.}
This family can have an effectively unbounded stability horizon and constant
device memory. Its general access horizon is the recent window plus a small set
of fixed initial positions. It is appropriate when relevant information is
local, when continuous fluency matters, or when another subsystem handles
older state.

\paragraph{Recurrent and compressed state.}
Transformer-XL carries segment-level recurrent states
\citep{dai2019transformerxl}; Recurrent Memory Transformer, Block-Recurrent
Transformers, and Memorizing Transformers provide distinct learned-state or
retrieval recurrence mechanisms
\citep{bulatov2022rmt,hutchins2022blockrecurrent,wu2022memorizing}.
State-space models such as Mamba provide another bounded recurrent path. These systems create a causal path from old content to
future outputs, so their access horizon can exceed the visible token window.
Their failure mode is lossy interference: bounded state must decide what to
preserve before future queries are known. \threeh's update, interference, and
delayed-query tests target this regime.

\paragraph{Retrieval-augmented memory.}
Retrieval-augmented language models store content externally and fetch a subset
at query time \citep{guu2020realm,lewis2020rag,borgeaud2022retro,
khandelwal2020knnlm}. InfLLM combines a local window with block-level context
memory for evicted states \citep{xiao2024infllm}. Their physical access horizon
can be as long as the store's retention period, but their utility horizon
depends on indexing, query formation, selection, and downstream use. Storage
and retrieval costs belong in the budget.

\paragraph{Hybrid systems.}
A practical long-running agent may use all three: attention sinks for stable
local generation, recurrent state for compact working memory, and retrieval for
sparse historical evidence. The three horizons suggest modular objectives.
Optimize the sink-window layer for stability, the write path for retention,
and retrieval plus execution for utility. A single end-to-end score cannot
identify which module failed.

\section{Implications for Claims and Design}

\paragraph{Reserve ``infinite'' for the measured property.}
``Infinite-duration generation with bounded KV cache'' is a precise and strong
claim. ``Infinite context'' suggests arbitrary historical access and should
require retention evidence beyond the active window. ``Long-term memory''
should additionally require delayed task utility.

\paragraph{Do not train an anchor and call it memory.}
A dedicated sink token may make streaming behavior easier to learn
\citep{xiao2024streamingllm}. Unless its state is updated with later content,
however, it remains an anchor. Semantic memory needs a write mechanism, an
external store, or both.

\paragraph{Match training to the horizon objective.}
Local next-token loss can reward stream stability without rewarding remote
retention. Training for memory requires losses whose targets depend on delayed
variables, preferably under interference and updates. Training for utility
must additionally expose the model to decisions that use retrieved or
compressed evidence.

\paragraph{Use failure decomposition.}
When an old fact does not affect a decision, the cause may be eviction,
compression loss, retrieval failure, or failure to use available evidence.
Counterfactual retention pairs identify whether any information path remains;
oracle retrieval separates access from execution; full-context controls
separate memory failure from intrinsic task difficulty.

\section{Related Work}

StreamingLLM motivates our central distinction and explicitly separates
streaming generation from context extension \citep{xiao2024streamingllm}.
Position extrapolation methods seek stable attention beyond training lengths
\citep{press2022alibi,su2024roformer,peng2024yarn,ding2024longrope};
efficient-attention methods reduce exact or approximate attention cost
\citep{dao2022flashattention,zaheer2020bigbird,jiang2024minference,
xionglong}; and
KV-cache policies retain or selectively load influential states
\citep{zhang2023h2o,li2024snapkv,liu2023scissorhands,ge2024fastgen,
tang2024quest,xiao2025duoattention}. Our argument is orthogonal: whatever
mechanism stabilizes inference, semantic access and task utility require
separate evidence.

Long-context benchmarks measure retrieval, aggregation, and reasoning over
large finite inputs \citep{bai2024longbench,hsieh2024ruler}. Analyses such as
Lost in the Middle show that nominal context size overstates effective use
\citep{liu2024lostmiddle}. Long-term conversational-memory benchmarks evaluate
information spread across sessions \citep{maharana2024locomo,
wu2025longmemeval}. \threeh connects these traditions to streaming systems by
placing stability, causal retention, and utility on distinct lag curves with
shared resource accounting.

\section{Conclusion}

Attention sinks solve a real systems problem: they can keep bounded-cache
generation stable far beyond a model's training window. They are not, by
themselves, semantic memory. The distinction becomes precise when context
length is replaced by three reported quantities: a stability horizon, an
access horizon, and a utility horizon. \threeh turns that distinction into an
evaluation contract with matched budgets, lag sweeps, counterfactual controls,
and task-level decisions. The resulting vocabulary gives streaming, recurrent,
retrieval, and hybrid systems credit for what they actually achieve, and
makes the missing capability visible.

\section*{Limitations}

Exact horizon measurements depend
on the chosen stream, task family, score threshold, and reference system.
Counterfactual streams can become unnatural if edited carelessly. Perplexity is
tokenization-dependent, and external-memory accounting involves policy choices
about storage infrastructure. The inequality between access and utility is
conceptually useful but not directly observable for arbitrary neural states;
experiments can establish lower bounds on access and upper bounds under
specified interventions, not recover every causal dependency. Finally, the
three horizons do not cover all desirable properties, including privacy,
calibration, factuality, and robustness to malicious memory content.

\bibliography{references}

\end{document}